\documentclass{article}

\usepackage{microtype}
\usepackage{graphicx}
\usepackage{subfigure}
\usepackage{booktabs} %

\usepackage{hyperref}

\usepackage[accepted]{icml2023}

\usepackage{amsmath}
\usepackage{amssymb}
\usepackage{mathtools}
\usepackage{amsthm}

\usepackage[capitalize,noabbrev]{cleveref}

\input{math_commands.sty}

\theoremstyle{plain}
\newtheorem{theorem}{Theorem}[section]

\newtheorem{lemma}[theorem]{Lemma}

\theoremstyle{definition}

\theoremstyle{remark}
\newtheorem{remark}[theorem]{Remark}

\usepackage[textsize=tiny]{todonotes}

\icmltitlerunning{Variance-Covariance Regularization Enforces Pairwise Independence in Self-Supervised Representations}

\begin{document}

\twocolumn[
\icmltitle{Variance-Covariance Regularization Enforces Pairwise \\
           Independence in Self-Supervised Representations}

\icmlsetsymbol{equal}{*}

\begin{icmlauthorlist}
\icmlauthor{Gr\'egoire Mialon}{yyy}
\icmlauthor{Randall Balestriero}{yyy}
\icmlauthor{Yann LeCun}{yyy,comp,sch}
\end{icmlauthorlist}

\icmlaffiliation{yyy}{Meta}
\icmlaffiliation{comp}{Courant Institute, New York University}
\icmlaffiliation{sch}{Center for Data Science, New York University}

\icmlcorrespondingauthor{Gr\'egoire Mialon}{gmialon@meta.com}

\icmlkeywords{Machine Learning, ICML}

\vskip 0.3in
]

\printAffiliationsAndNotice{\icmlEqualContribution} %

\begin{abstract}
Self-Supervised Learning (SSL) methods such as VICReg, Barlow Twins or W-MSE avoid collapse of their joint embedding architectures by constraining or regularizing the covariance matrix of their projector's output. This study highlights important properties of such strategy, which we coin Variance-Covariance regularization (VCReg).
More precisely, we show that {\em VCReg combined to a MLP projector enforces pairwise independence between the features of the learned representation}. This result emerges by bridging VCReg applied on the projector's output to kernel independence criteria applied on the projector's input.
We empirically validate our findings where %
(i) we put in evidence which projector's characteristics favor pairwise independence, %
(ii) we demonstrate pairwise independence to be beneficial for out-of-domain generalization, (iii) we demonstrate that the scope of VCReg goes beyond SSL by using it to solve Independent Component Analysis.
This provides the first theoretical motivation and explanation of MLP projectors in SSL.

\end{abstract}

\begin{figure}[ht]
\begin{center}
\centerline{\includegraphics[width=.95\columnwidth]{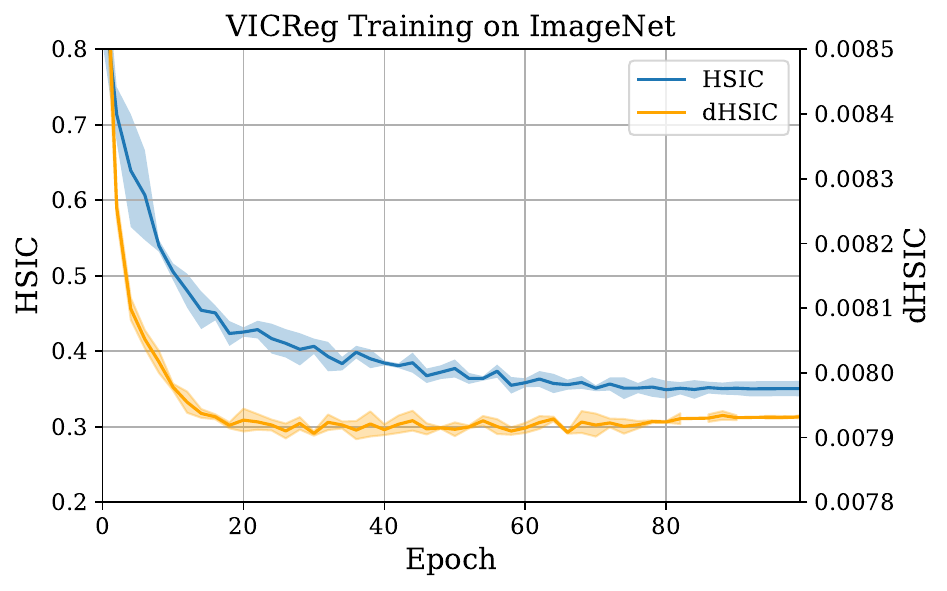}}
\caption{Pairwise independence (HSIC) of the features in the learned representation increases during training while mutual independence (dHSIC) stagnates: VCReg of the projector implicitly optimizes the former but not the latter. Averaged over three runs.}
\label{fig:indep_training}
\end{center}
\vspace{-1cm}
\end{figure}

\section{Introduction}

Self-Supervised Learning (SSL) via joint embedding architectures has risen to learn visual representations outperforming their supervised counterpart. This paradigm enforces similar embeddings for two augmented versions of the same sample, thus allowing an encoder $f$ to learn a representation for a given modality without labels. %
Importantly, $f$ could solve the learning task by predicting the same embedding for every input, a failure mode known as collapse. To avoid this, various mechanisms have been proposed hence the diversity of SSL methods~(\textit{e.g.}, \citet{grill2020bootstrap, caron2020unsupervised, caron2021emerging, chen2021exploring}). Most of these are a composition $g \circ f$ of the encoder with a projector neural network $g$. Only $f$ is retained after training, in opposition to supervised training that never introduces $g$. The projector was proposed by~\citet{chen2020simple} and significantly improved the quality of the learned representation in terms of test accuracy on ImageNet and other downstream tasks. Although some works~\citep{appalaraju2020towards,bordes2022guillotine} provide empirical knowledge on the projector, none provide a theoretical analysis of MLP projectors in practical SSL \citep{jing2021understanding,huang2021towards,wang2020understanding,haochen2021provable,tian2020makes, wang2021understanding,cosentino2022toward}.

This study sheds a new light on the role of the projector via the lens of Variance-Covariance Regularization (VCReg), a strategy introduced in recent SSL methods~\citep{bardes2021vicreg, zbontar2021barlow, ermolov2021whitening} to cope with collapse by constraining or regularizing the covariance or cross-correlation of the projector $g$ output to be identity. More precisely, we demonstrate that \textit{VC regularization of the projector's output enforces pairwise independence between %
the components of the projector's input i.e. the encoder's output}, and connects this property to projector's characteristics such as width and depth. %
This provides the first theoretical motivation and explanation of MLP projector in SSL:
fully or partially pairwise independent representations are generally sought for, \textit{e.g} to disentangle factors of variation~\citep{trauble2021disentangled}. In particular,~\citet{li2019learning} demonstrates that factors in real-world data tend to be pairwise independent.
Our experimental analysis suggests that different levels of pairwise independence of the features in the representation emerge from a variety of SSL criteria along with mutual independence. However, as opposed to other frameworks, VCReg allows for theoretical study and explicit control of the learned independence amount. We prove and experimentally validate this property for random projectors, study how it emerges in learned projectors, and demonstrate pairwise independence to be beneficial for out-of-domain generalization. %
We then ablate the SSL context and lean on our findings to show that VCReg of a SSL projector solves Independent Component Analysis (ICA). %
Beyond providing a novel theoretical understanding of the projector, we believe that this work also leads to a better understanding of VICReg. The scope of VCReg is not limited to SSL: our experiments on ICA open the way to other applications where some degree of independence is needed.

\section{Background}

\subsection{Measuring Statistical Independence Using Kernel Methods}
\label{sec:background1}

Measuring the independence between two sets of realizations $\{\mX_1^1,\dots,\mX_1^N\},\{\mX_2^1,\dots,\mX_2^N\}, \mX_1^i \in \mathbb{R}^{M},\mX_2^i \in \mathbb{R}^{M}$ is a fundamental task that has a long history in statistics e.g. through the Mutual Information (MI) of the two random variables $X_1$ and $X_2$ from which those two sets are independently drawn from \citep{cover1999elements}. Those variables are said independent if the realization of one does not affect the probability distribution of the other. 
Computing the MI in practice is known to be challenging \citep{goebel2005approximation}, which has led to considerable interest in using alternative criteria e.g. based on functions in Reproducing Kernel Hilbert Spaces (RKHS) \citep{bach2002kernel}, a special case of what is known as functional covariance or correlation \citep{renyi1959measures}. It consists in computing those statistics after nonlinear transformation \citep{leurgans1993canonical} as in
\begin{equation}
\sup_{f_1\in \mathcal{F}_1,f_2\in\mathcal{F}_2} \Corr \langle f_1(X_1), f_2(X_2)\rangle,\label{eq:coco}
\end{equation}
where $f_1,f_2$ are constrained to lie within some designed functional space, and the $\Cov$ can be used instead of the $\Corr$. If \cref{eq:coco} is small enough, then $X_1$ and $X_2$ are independent in regard to the functional spaces $\mathcal{F}_1,\mathcal{F}_2$. For example, if $\mathcal{F}_1$ and $\mathcal{F}_2$ are unit balls in their respective vector spaces, then \cref{eq:coco} is just the norm of the usual correlation/covariance operator \citep{mourier1953elements} which would be enough for independence  under joint Gaussian distributions \citep{melnick1982misspecifications}.

\paragraph{HSIC and pairwise independence.}
More recently, \citet{gretton2005measuring} introduced a pairwise independence criterion known as the Hilbert-Schmidt Independence Criterion (HSIC) between the two random variables $X_1$ and $X_2$ which can be estimated given empirical samples $\mX_1$ and $\mX_2 \in \mathbb{R}^{N \times M}$ via
\begin{equation}
    \HSIC(\mX_1,\mX_2) := \frac{1}{(N-1)^2}\Tr(\mK_1\mH\mK_2\mH),\label{eq:HSIC}
\end{equation}
with $\mH$ the centering matrix $\mI-\mathbf{1}\mathbf{1}^T\frac{1}{N}$, $(\mK_1)_{i,j}=k_1(\mX_1^i,\mX_1^j)$ and $(\mK_2)_{i,j}=k_2(\mX_2^i,\mX_2^j)$ the two kernel matrices of $\mX_1$ and $\mX_2$ respectively, and $k_1$, $k_2$ of $\mathcal{F}_1,\mathcal{F}_2$ universal kernels such as the Gaussian kernel (see \citet{steinwart2001influence,micchelli2006universal} for other examples).
Crucially, since 
\begin{equation}
    \HSIC(X_1,X_2)\geq \sup_{f_1\in \mathcal{F}_1,f_2\in\mathcal{F}_2} \Cov \langle f_1(X_1), g(X_2)\rangle,
\end{equation}
HSIC can be used to test for (pairwise) independence as formalized below. 
\begin{theorem}[Thm. 4 from \citep{gretton2005measuring}]
$\HSIC(X_1,X_2)=0$ if and only if $X_1$ and $X_2$ are independent. %
\end{theorem}
\cite{gretton2005measuring} also provide a statistical test for pairwise independence based on HSIC. Further quantities such as upper bounds on the MI can be found in a similar way, e.g. see Thm.~16 in \cite{gretton2005kernel}. In our experiments, we rely on HSIC under the Gaussian kernel scaled by the median of the distribution of pairwise euclidean distances between samples. To compare HSIC across different models, we further normalize following~\citet{kornblith2019similarity}.

\paragraph{dHSIC and mutual independence.}
Mutual independence of a set of $D$ $M$-dimensional random variables $X_1, \dots, X_D$ is a stronger property than independence between all pairs of random variables in the set. To evaluate it, \citet{pfister2018kernel} introduce dHSIC, a multivariate extension of HSIC. In short, dHSIC$(X_1,\dots, X_D)$ measures the distance between mean embeddings $\mu$ under a RKHS $\mathcal{F}$~\citep{smola2007hilbert} of the product of distributions and the joint distribution
\begin{equation}
    \|\mu(\mathbb{P}^{X_1} \otimes \dots \otimes \mathbb{P}^{X_D}) - \mu(\mathbb{P}^{X_1, \dots, X_D}) \|_{\mathcal{F}}.
\label{eq:dHSIC_kernel}
\end{equation}
Similarly to HSIC, \citet{pfister2018kernel} establish the equivalence between %
$\text{dHSIC} = 0$
and mutual independence along with a statistical test. We provide an estimator of dHSIC given empirical samples $\mX_1, \dots, \mX_D$ of the above random variables in $\mathbb{R}^{N \times M}$ each as well as implementations of HSIC and dHSIC in Appendix~\ref{app:implem}.

\paragraph{Complexities.} In what follows, we consider the $D$ features of a batch of representations $\mX \in \mathbb{R}^{N \times D}$ as $D$ scalar random variables ($M=1$) with $N$ samples each. The subsequent complexities of HSIC for one pair of features and dHSIC are $\mathcal{O}(N^2)$ and $\mathcal{O}(DN^2)$ respectively. Testing pairwise independence between all variables in a $D$-set with HSIC is $\mathcal{O}(D^2N^2)$ while dHSIC requires $N \geq 2D$~\citep{pfister2018kernel}. Since competitive visual representations typically have $D=2048$, we resort to surrogates to estimate independence of the representations in practice. These surrogates are detailed in Section~\ref{sec:exp}.

\subsection{Variance-Covariance Regularization in Self-Supervised Learning}

\textbf{SSL with joint embeddings} %
learns visual representations by producing two different augmented views of an input batch of images $\mS \in \mathbb{R}^{N \times W \times H}$, denoted by $\mS_{\rm left}$ and $\mS_{\rm right}$. Each view is fed to an encoder $f$, typically a ResNet50~\citep{he2016deep}, producing representations $\mX_{\rm left}$ and $\mX_{\rm right} \in \mathbb{R}^{N \times D}$ which are passed through a projector $g$ to output embeddings  $\mZ_{\rm left}$ and $\mZ_{\rm right} \in \mathbb{R}^{N \times P}$. Finally, an invariance term encouraging $\mZ_{\rm left}$ and $\mZ_{\rm right}$ to be similar is applied. While most SSL methods require architectural or training strategies to avoid collapse~\citep{grill2020bootstrap, he2021masked}, \citet{bardes2021vicreg} and \citet{zbontar2021barlow} only require to modify the loss. After training, only $f$ is retained to be used in downstream tasks. We will denote the $(2N,P)$ matrix $\mZ_{\rm total} \triangleq [\mZ_{\rm left}^T,\mZ_{\rm right}^T]^T$. More background on self-supervised learning for images in theory and in practice can be found in Appendix~\ref{app:background}. In particular, we detail the crucial role of the projector in many SSL frameworks.

\paragraph{VICReg.} In~\citet{bardes2021vicreg}, an anti-collapse term $\mathcal{L}_{\rm VC}$, which we coin VC regularization (VCReg), is added to an invariance loss to form $\mathcal{L}_{\rm VIC}$:
\begin{align}
\mathcal{L}_{\rm VC}= \sum_{k=1}^{P}\max\left(0,1-\sqrt{\Cov(\mZ_{\rm total})_{k,k}}\right)\hspace{-0.08cm}\nonumber\\ \hspace{-0.1cm} + \alpha \sum_{j=1,j\not = k}^{P}\hspace{-0.15cm}\Cov(\mZ_{\rm total})^2_{k,j}\label{eq:VCReg},\\
\mathcal{L}_{\rm VIC}= \frac{1}{N} \sum_{n=1}^{N}\|(\mZ_{\rm left})_{n,.}-(\mZ_{\rm right})_{n,.}\|_2^2 + \mathcal{L}_{\rm VC} %
\label{eq:VICReg},
\end{align}
The leftmost term in $\mathcal{L}_{\rm VC}$ corresponds to regularizing the variance of each feature in $\mZ_{\rm total}$ to be at least unit, while the second term seeks to minimize the covariance between each pair of features in $\mZ_{\rm total}$. Note that $\mathcal{L}_{\rm VC}$ applies to each view separately. In~\citet{bardes2021vicreg}, the weight of each term in $\mathcal{L}_{\rm VIC}$ can be tuned. Since the authors find that best results are obtained with equal weights for the invariance and the variance terms, we only vary $\alpha$. Note that \citet{bardes2021vicreg, zbontar2021barlow} observe that wider projectors further improve the representation learned by the encoder $f$, yet neither further study this intriguing phenomenon. Although this work focuses on VCReg as formulated in VICReg, we provide background on Barlow Twins and W-MSE in the Appendix~\ref{app:background}, while next section establishes the similarity between VICReg, Barlow Twins and W-MSE as VCReg optimizers.

\section{VC Regularization of SSL Projector's Output Enforces Pairwise Independent Features at the Encoder's Output}
\label{sec:result}

In this section, we demonstrate that for random Multi-Layer Perceptrons (MLPs) projectors, minimizing VC of the projector's output amounts to minimizing HSIC ---a measure of pairwise dependence (see \cref{sec:background1})--- between all pairs of features
in the learned representation, \textit{i.e} the projector's input. The randomness assumption has already been used to model weight evolution during training~\citep{franchi2020tradi, blundell2015weight}. We then justify how this reasoning extends to learned projectors. Our claims are experimentally verified in \cref{sec:exp}. 
\vspace{-.5cm}
\paragraph{Notations.} In our setting, $\mX \in \mathbb{R}^{N \times D}$ is the $D$ dimensional output of the encoder $f$ for a batch of size $N$. $\mX$ is then fed to the projector $g$ to form embeddings $\mZ= g(\mX) = [g\mX_1,\dots,g\mX_P] \in \mathbb{R}^{N \times P}$ on which SSL criteria are generally applied. Note that $\mZ$ would be $\mZ_{\rm left}$ or $\mZ_{\rm right}$ in the previous subsection. The acronym MLP refers to projectors typically used in SSL, a neural network with three layers of same width, unless stated otherwise.

Our proof strategy consists in proving that VCReg of each MLP layer (Linear + Batch Normalization (BN) + ReLU) output enforces pairwise independence of its input, before composing these results. We first study nonlinear elementwise projectors $g:\mathbb{R} \mapsto \mathbb{R}^L$, which belong to the wider class of DeepSets~\citep{zaheer2017deep}, and of which BN followed by ReLU can be seen as an instance. We denote the mapping of such projectors as $\mZ = g(\mX) = [g(\mX_{:,1}), \dots, g(\mX_{:,P})]$.

\begin{lemma}[Nonlinear elementwise projectors minimize HSIC of their input] Let $g:\mathbb{R} \mapsto \mathbb{R}^L$ be a nonlinear elementwise projector; then, minimizing the covariance of $\mZ$ with respect to the encoder $f$ amounts to minimizing HSIC on all feature pairs in $\mX$ under the kernels $\mK_{i} = g(\mX_{:,i})g(\mX_{:,i})^T$.
\label{prop:nonlinar_proj}
\end{lemma}
\vspace{-1cm}
\begin{proof}
Let us consider the $N$ values of the $i^{\rm th}$ data feature ($\mX_{:,i}$) as realizations of a random variable. Recalling \cref{eq:HSIC}, independence of two random variables can be estimated via HSIC. 
Considering the arbitrarily complicated network $g$ and $\mZ=[g(\mX_{:,1}),\dots,g(\mX_{:,D})] \in \mathbb{R}^{N \times DL}$, we have:
\begin{align}
\HSIC(&\mX_{:,i}, \mX_{:,j}) = \nonumber\\
&\;\;\;\;\frac{\Tr(g(\mX_{:,i})g(\mX_{:,i})^T\mH g(\mX_{:,j})g(\mX_{:,j})^T\mH)}{(N-1)^2}\nonumber\\
&= \frac{1}{(N-1)^2} \left\|g(\mX_{:,i})^T\mH g(\mX_{:,j})\right\|_F^2\nonumber\\
&=\left\|\Cov\left(g(\mX_{:,i}),g(\mX_{:,j}\right)\right\|_F^2\nonumber\\
&=\left\|\Cov\left(\mZ\right)_{1+iL:1+(i+1)L,1+jl:1+(j+1)L}\right\|_F^2\nonumber
\end{align}
Hence, 
\begin{align*}
\sum_{i\not = j}\HSIC(\mX_{:,i}, \mX_{:,j}) 
\hspace{-0.08cm}=\hspace{-0.08cm} 
\left\|\Cov\hspace{-0.08cm}\left(\mZ\right)\hspace{-0.08cm}\odot\hspace{-0.08cm} \left(\hspace{-0.08cm}(1-\mI_{D}) \hspace{-0.08cm}\otimes\hspace{-0.08cm} \mathbf{1}_{L}\mathbf{1}_{L}^T\right)\right\|_F^2,
\end{align*}
and, in the case $L=1$ we have $\mathbf{1}_{L}\mathbf{1}_{L}^T=1$ leading to $\sum_{i\not = j}\HSIC(\mX_{:,i}, \mX_{:,j})=\sum_{i\not = j}\Cov\left(\mZ\right)^2_{i,j}$, concluding the proof.
\end{proof}
To rigorously obtain independence, the $\mK_{i}$'s must be universal kernels. %
In practice, Batch Normalization in the projector can be considered random as it uses the batch statistic. Combining this operation with a ReLU, we obtain random elementwise nonlinearities approaching random features~\citep{rahimi2007random} of a universal kernel~\citep{sun2018approximation}.
Increasing $L$ improves the approximation of such
kernel \citep{chen2017relative} i.e. the larger $L$, the better approximation of HSIC the covariance term in VICReg is.

\begin{remark}[Necessity of variance regularization] Although the variance regularization term does not explicitly appear when minimizing HSIC on all pairs, it is necessary when optimizing $\mX$ to prevent the degenerate solution of $\mX$ being a constant, a common collapse mode of SSL.
\end{remark}

\begin{lemma}[Random linear projectors minimize HSIC of their input] Let $g$ be a random linear projector with weights $\mW$, and $\mX$ has same variance for each column. Then, for large projectors, minimizing the covariance of $\mZ= g(\mX) = \mX\mW$ with respect to the encoder $f$ amounts to minimizing HSIC with a linear kernel for each pair of features %
in $\mX$. (Proof in \cref{app:proofs}.)
\label{prop:linar_proj}
\end{lemma}

Since the corresponding kernel in \cref{prop:linar_proj} is linear, only decorrelation can be achieved for such projector's input. This however differs from PCA, as we optimize over $\mX$ and not $g$'s parameters ($\mW$).
Proving \cref{prop:linar_proj} requires a projector with orthogonal weights, i.e. $\mW^T\mW=\mI$, which gets more and more accurate with random weights $\mathcal{U}(-\nicefrac{1}{\sqrt{D}}, \nicefrac{1}{\sqrt{D}})$ as $P$ (the output dimension of $g$) increases. This follows since the central limit theorem states that the dot-product between two $P$-dimensional weight vectors tends to $0$ with rate $\mathcal{O}(\nicefrac{1}{\sqrt{P}})$. Weight initialization in neural nets roughly follows  $\mathcal{U}(-\nicefrac{1}{\sqrt{D}}, \nicefrac{1}{\sqrt{D}})$, which will be used to instantiate random projectors.

\begin{theorem}[MLP projectors with random weights enforce pairwise independence.]
Let us consider a MLP composed of alternating random linear layers and elementwise nonlinearities. Then, for large projectors, minimizing the variance and covariance of the output $\mZ$ enforces pairwise independence between all pairs of features %
in the input $\mX$.
\label{cor:composition}
\end{theorem}

\begin{proof}
Let us consider the last block of such MLP, \textit{i.e.} a fully-connected linear layer followed by an elementwise nonlinearity. According to \cref{prop:nonlinar_proj}, applying VCReg to the MLP, \textit{i.e.} at the output of the nonlinearity will enforce pairwise independence under corresponding kernel for the input of the nonlinearity, which is also the output of the last linear layer. %
If the latter is wide enough so that it can be considered orthogonal, Theorem 11 in~\cite{comon1994independent} ensures that pairwise independence is preserved for the input of the layer. We can then recursively extend the result backward to the whole MLP. If the last MLP layer is a fully-connected linear, then~\cref{prop:linar_proj} applies and we go back to the preceding elementwise nonlinearity.
\end{proof}
\cref{fig:implicit_reg_random} in \cref{app:more_hsic} shows that each hidden layer in the random MLP is recursively and implicitly VC-regularized from VCReg being applied only at the projector's output. Following \cref{cor:composition}, we expect wider projectors to better enforce pairwise independence while adding layers or learning the projector is not necessary; see \cref{sec:exp} for empirical validation.

\paragraph{Extension to BarlowTwins and W-MSE, and generality of VCReg.}~Our results focus on VCReg as formulated in VICReg but can in fact be extended to methods that constrain the covariance of $\mZ$ explicitly, namely BarlowTwins and W-MSE \citep{balestriero2022contrastive}. Indeed, the objective in W-MSE (Equations 3-4 in~\citet{ermolov2021whitening}) is VICReg with explicit constraint on the variance and covariance. %
Increasing the variance and covariance hyper-parameters in VICReg produces W-MSE hence our results extend seamlessly. The objective from~\cref{eq:BT} in BarlowTwins is also similar to VICReg. The derivation, deferred to \cref{proof:BT}, shows that minimizing the constrained form of BarlowTwins objective from~\cref{eq:BT} is equivalent to minimizing VICReg's invariance term whilst explicitly constraining the variance covariance terms as in W-MSE; hence our results also hold for BarlowTwins; see \cref{sec:exp} for empirical validation. In particular we will see that as BarlowTwins %
explicitly enforces minimum VCReg, it better optimizes HSIC compared to VICReg with standard hyper-parameters. Finally, as opposed to BarlowTwins loss and most SSL methods, VCReg can be used and be beneficial within single branch architectures. We provide such use case in \cref{sec:exp}.

\paragraph{Learned projectors.}~In state-of-the-art SSL representations, the projector is learned, which is not rigorously covered by \cref{cor:composition}.
Complementing the study of~\cite{bordes2022guillotine}, we argue that learning the projector is only crucial to satisfy the invariance criterion since random projectors are sufficient to obtain pairwise independent features, as demonstrated in \cref{sec:exp}. In fact, our experiments show that (i) using VICReg, keeping the projector random yield representations with low HSIC, (ii) using VICReg, learning the projector to optimize VCReg more strongly than the invariance term reduces performances, and (iii) using VCReg without an invariance term, learning the projector (e.g. in the later studied ICA setting) creates a degenerate representation that does not enforce pairwise independence.
We thus conjecture that learning the projector's parameters to mainly minimize VICReg's invariance term leaves the parameters close enough to their random initialization \citep{jacot2018neural} to maintain an accurate estimate of HSIC. In fact, we will see that the wider the projector, the less far away from initialization the parameters have to move, and the better optimized HSIC. Importantly, to better mimic the behavior of a learned projector, we will also conduct experiments where the random projector is resampled at each optimization step.
\section{Related Work}

\paragraph{Feature decorrelation,} or whitening, ensures that correlation between each pair of different features in a batch of feature vectors is zero, and that each feature has unit variance. It was originally used as a data pre-processing technique, see e.g.~\cite{hyvarinen2000independent}, before being extended to deep networks~\citep{cogswell2015reducing} as a regularizer. In the context of SSL, \citet{hua2021feature,ermolov2021whitening} find that feature decorrelation helps solving the collapse issue. The former avoids collapse via appropriate Batch Normalization and its decorrelating variant~\citep{huang2018decorrelated} in the projector. This variant of Batch Normalization can be seen as the hard constraint counterpart of BarlowTwins. Practically, whitening can be implemented by learning a fully-connected layer~\citep{husain2019remap} recovering Principal Component Analysis; as mentioned earlier (Section~\ref{sec:result}), this differs from VCReg which keeps the layer's parameters random and optimizes its input.

\paragraph{Independence criterion for learning features.} Enforcing mutual independence to learn a representation has been proposed for example by~\citet{schmidhuber1992learning}. More recently, and in the context of supervised learning, \citet{chen2019rethinking} demonstrated improved training of ResNets by reducing the pairwise dependence in the features at each layer via a combination of Dropout and Batch Normalization. By opposition, VCReg is applied to the projector's output and common SSL projectors do not rely on Dropout; we also show that Batch Normalization is not necessary to reduce pairwise dependence.

\paragraph{HSIC in supervised and self-supervised learning.} HSIC-based losses have been employed in supervised learning e.g. see \citet{mooij2009regression,greenfeld2020robust} to ensure independence between the residual errors and the labels of a task at hand. \citet{kornblith2019similarity} measure similarity between two neural representations via HSIC. In SSL, \citet{tsai2021a} showed that a modified version of BarlowTwins loss maximizes HSIC under a linear kernel between the embeddings of two augmented versions of the same sample. In a similar fashion, \citet{li2021self} proposed a SSL framework based on maximizing the dependence between the embeddings of transformations of an image and the image identity via HSIC. Both works differ from ours, which connects VCReg to the minimization of HSIC between pairs of features in the representation.
\section{Experiments}
\label{sec:exp}

We first put in evidence pairwise independence properties emerging in learned visual representations, before validating our theoretical findings from Section~\ref{sec:result}. Then, we demonstrate low HSIC to be beneficial for in-domain and out-of-domain generalization. Finally, we ablate the SSL context and demonstrate that the composition of a SSL-like random projector with VCReg induces enough independence to perform ICA. Importantly, none of the experiments with VCReg use Dropout, and Batch Normalization is not used when the projector is random. Hence, pairwise independence cannot be attributed to those two techniques as opposed to~\citet{chen2019rethinking, hua2021feature}.

\subsection{Pairwise independence emerges in most visual representations}
\label{exp:analysis}

\begin{table*}[ht!]
    \small
    \centering
    \caption{Pairwise dependence (HSIC), pairwise independence testing, and test accuracy for SSL or supervised representations averaged on multiple subsets of features, with internal variance. VCReg methods, including SimCLR~\citep{garrido2022duality}, further improve HSIC.}
    \begin{tabular}{l|c|c|c|c} \toprule
    {Method (100 epochs)} & {Normalized HSIC $\downarrow$} & {\% of independent pairs $\uparrow$} & {ImageNet Top1 $\uparrow$} \\
    \midrule
    {BarlowTwins~\citep{zbontar2021barlow}}  & 0.006 $\pm$ 0.004 & 38 $\pm$ 2 & 67.8 \\
    {SimCLR~\citep{chen2020simple}}  & 0.008 $\pm$ 0.008 & 36 $\pm$ 1 & 67.7 \\
    {VICReg~\citep{bardes2021vicreg}}  & 0.009 $\pm$ 0.005 & 34 $\pm$ 2 & 68.1 \\
    {DINO~\citep{caron2021emerging}} & 0.01 $\pm$ 0.01 & 35 $\pm$ 3 & 70.4 \\
    \midrule
    {Supervised~\citep{wightman2021resnet}} & 0.010 $\pm$ 0.008 & 33 $\pm$ 3 & 78.1 \\
    \bottomrule
    \end{tabular}
    \label{fig:indep_table}
\vspace{-.5cm}
\end{table*}

\paragraph{Setup.} In these experiments, we track two metrics for statistical independence of the components in the learned representation during ResNet50 training on ImageNet with popular SSL frameworks, as well as a supervised baseline. The first metric, HSIC~\citep{gretton2005measuring}, tracks pairwise independence. The second metric, dHSIC~\citep{pfister2018kernel}, tracks mutual independence. As explained in~\cref{sec:background1}, neither HSIC nor dHSIC scale to the full representation (2048 components). Therefore, we instead rely on proxys: for pairwise independence, we compute and average Normalized HSIC~\citep{kornblith2019similarity} on all pairs of the first $10$ components of the representation. For mutual independence, we sample sets of $10$ components for which we display dHSIC. We compute these statistics on the same batch of $1000$ samples (one per class) from the ImageNet validation set unless stated otherwise. A decrease of these metrics throughout the training would suggest decreasing dependence among the representations. Although statistical tests are available for both metrics, we track bare HSIC and dHSIC as they are continuous. However, we perform HSIC tests in this first serie of experiments so that the bare values can be linked to portions of independent pairs of components in the representation. See Appendix~\ref{app:exp_details} for our detailed setup.

\paragraph{Results.} Figure~\ref{fig:indep_training} suggests that VICReg implicitly optimizes pairwise independence of the features %
in the representation throughout the training (HSIC, left), but not mutual independence (dHSIC, right). We will therefore not study dHSIC in subsections~\ref{exp:val_theory} and \ref{exp:generalization}, while an ablation in subsection~\ref{exp:ica} further confirm that VCReg cannot optimize dHSIC.
\cref{fig:indep_table} shows that VCReg frameworks (VICReg, Barlow Twins) and related (SimCLR\footnote{~\citet{garrido2022duality} demonstrate that up to implementation details, VICReg and SimCLR optimize the same objective and can in fact achieve the same performance with careful hyperparameter tuning.}) achieve lower HSIC than DINO and Supervised, and Barlow Twins better enforces low HSIC since its covariance matrix is constrained, as expected from~\cref{sec:result}.
Precise values along with HSIC independence tests at level $\alpha=0.5$ and corresponding test accuracies can be found in Table~\ref{fig:indep_table}. Lower HSIC values do not necessarily entail more successful tests: it is possible to allocate low HSIC to sufficiently many pairs while having an overall larger HSIC. For example, Supervised has less independent pairs in spite of an overall smaller HSIC than DINO.
Table~\ref{fig:indep_table} also suggests that pairwise independence is not perfectly predictive of test accuracy. Although it seems to be a desirable property, it is not sufficient to yield a good representation: for example, DINO achieves better test accuracy than BarlowTwins in spite of higher HSIC. \cref{exp:generalization} will however show and discuss how HSIC is correlated with in domain and out of domain downstream accuracy.

\subsection{Projector characteristics fostering or hurting pairwise independence}
\label{exp:val_theory}

\paragraph{Setup.} These experiments validate the results from Section~\ref{sec:result} by illustrating which projector characteristics help or hurt HSIC optimization when learning visual representations. The setup is the same as above except that dHSIC is not measured. %
However, the projector can be random with resampling at each optimization step or without, as in~\cref{cor:composition}, or learned as in~\cref{exp:analysis}. We apply an invariance loss along with VCReg on the output of the projector and scale the covariance coefficient with its width $P$ for fair comparison since the covariance term in VCReg grows quadratically with the projector output dimension (see~\cref{eq:VCReg}). See Appendix~\ref{app:exp_details} for our detailed setup.

\paragraph{Results.} Figure~\ref{fig:hsic_vs_dim} supports the theory from~\cref{sec:result}: low HSIC can be achieved by random projectors. Hence, learning the projector is not necessary to obtain pairwise independence. As predicted in theorey, for all three settings, HSIC is better optimized for wider projectors. HSIC can be influenced by other factors of variations: as conjectured in~\cref{sec:result}, for random projectors, increasing depth hurts HSIC, see~\cref{fig:hsic_vs_depth} in Appendix~\ref{app:more_hsic}. Increasing the covariance coefficient in VCReg also improves HSIC, albeit slightly at the expense of downstream accuracy, see Table~\ref{fig:cov_coeff_hsic} in Appendix~\ref{app:more_hsic}.
\begin{figure}
\begin{center}
    \centerline{\includegraphics[width=.95\columnwidth]{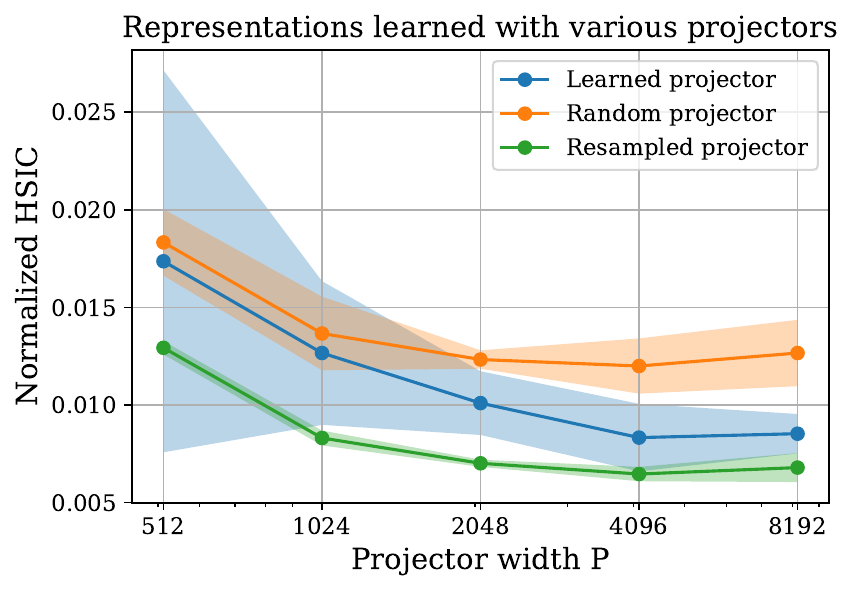}}
    \caption{Normalized HSIC (computed on ImageNet validation) for various representations learned with different projectors width and setup: learned, random, random and resampled at each optimization step (3 runs each). Wider projectors and resampling both yield more pairwise independence in the learned representation.}
    \label{fig:hsic_vs_dim}
\end{center}
\vspace{-1cm}
\end{figure}
We observe that (i) random resampled projectors from~\cref{sec:result} learn representations with better pairwise independence properties than learned projectors --we will heavily use this property in~\cref{exp:ica}--, and (ii) random projectors do worse. (i) can be explained by the fact that learning the projector to optimize VCReg is detrimental to HSIC since it could achieve trivial solutions such as $Z = g(X) = (X_1, 0, \dots, 0)$: the presence of the invariance term explains why a learned projector does not collapse on such solutions. (ii) can be explained as follow: since the projector weights are fixed, the encoder only learns to decorrelate a limited set of nonlinear functions as opposed to the random resampled or learned projectors, whose weights move throughout the training. 
Although random resampled projectors achieve lower HSIC, most of them are not as competitive as learned projectors in terms of in domain test accuracy of the learned representation, see~\cref{fig:hsic_downfall}. This is not surprising, as there is no reason for the best $D=2048$ representations for ImageNet to have all their features pairwise independent. In fact, this suggests that optimizing too much HSIC can be detrimental for the quality of the representation: a possible explanation is that it is not possible to extract all the relevant information from the training set with respect to classification while preserving total pairwise independence in the representation. This complements the view of~\cite{appalaraju2020towards,bordes2022guillotine}, which argue that the projector requires some learning capacity to filter out information that is irrelevant to fulfill the invariance loss. 
Finally,~\cref{fig:hsic_vs_depth} in Appendix~\ref{app:add_exps} shows that adding layers to random projectors is detrimental to HSIC: implicit VCReg of the activations, an assumption of~\cref{cor:composition}, will be more and more loosely enforced. These experiments suggest that the projector capacity should rather be increased via width than via depth: adding layers can be detrimental to HSIC as seen above but also to test accuracy~\citep{appalaraju2020towards,chen2021intriguing}.
\vspace{-.6cm}
\begin{table}[h]
    \small
    \centering
    \caption{Normalized HSIC (divided by 1e3 for readability) and Top1 accuracy on ImageNet (both computed on validation) for representations learned with different projector types ($P=8192$).}
    \begin{tabular}{lccc} \toprule
    {Projector type} & Random & Learned & Resampled \\
    \midrule
    {Norm. HSIC ($\downarrow$)} & 13 $\pm$ 2 & 9 $\pm$ 1 & 6.8 $\pm$ 0.1  \\
    \midrule
    {Top1 ImageNet ($\uparrow$)} & 54.1$\pm$0.1 & 68.12$\pm$0.01 & 52.5$\pm$0.2 \\
    \bottomrule
    \end{tabular}
\label{fig:hsic_downfall}
\end{table}
\vspace{-.7cm}
\subsection{HSIC is correlated to in-domain and out-of-domain downstream accuracy}
\label{exp:generalization}
In this experiment, we plot the Normalized HSIC and Relative Top1 accuracy on various datasets for the representations trained in this work, \textit{i.e.}, for various projector configurations, see~\cref{fig:hsic_vs_acc}. Overall, HSIC is correlated to in-domain and out-of-domain downstream accuracy. As expected from~\cref{fig:hsic_downfall}, an elbow appears around Normalized HSIC$=0.008$: below this threshold, the representation gains pairwise independence at the expense of its information content, a form of collapse. Crucially, HSIC does not require labels to be computed hence could be used for model selection in SSL.

\begin{figure}
\begin{center}
\centerline{\includegraphics[width=.95\columnwidth]{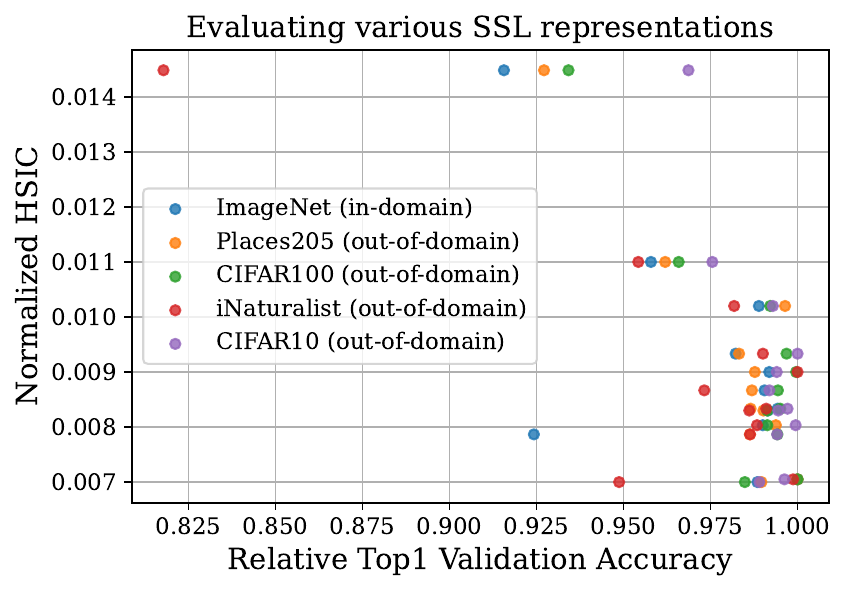}}
\caption{\small Normalized HSIC (computed on ImageNet validation) of representations correlates with downstream accuracy both in domain and out of domain. To each HSIC level corresponds a representation. For each dataset, the accuracies were rescaled with respect to their maximum for better readability.}
\label{fig:hsic_vs_acc}
\end{center}
\vspace{-1cm}
\end{figure}

\subsection{Ablation: Independent Component Analysis (ICA) with VCReg}

To demonstrate that our findings hold outside of SSL, we show that VCReg of a SSL-like projector's output induces enough pairwise independence to solve linear ICA, \textit{i.e.} recovering independent sources $\mS \in \mathbb{R}^{N \times D}$ from a mixture $\mY = \mS \mA \in \mathbb{R}^{N \times D}$. See Appendix~\ref{app:background} for an introduction to ICA.
\vspace{-.5cm}
\paragraph{VC regularized random projectors solve linear ICA, learned projector don't.}
\label{exp:ica}
In this setting, finding $\mM$ enforcing \textit{pairwise} independence in the components of $\mY\mM$ is generally sufficient to recover $\mS$~(\cite{comon1994independent}, Theorem 11). VCReg of a random projector's output should therefore be able to recover $\mS$.
Our model can be seen on Figure~\ref{fig:archi_linearica}: whitened batches of mixtures $\mY$ are fed to an encoder $f$. Here, $f$ is a linear transformation $\mM$ described above. The output $\mX = f(\mY)$ is then fed to a projector $g$, and VCReg is applied to the output $\mZ = g(\mX)$ covariance matrix. %
As in Experiments~\ref{exp:val_theory}, the projector is randomly resampled at each gradient step to get lowest possible HSIC. This setting corresponds to one branch of a VICReg network (i) with a linear projection $M$ encoder instead of a neural network, (ii) with a random projector, and (iii) without invariance criterion. We perform ICA on two datasets, a synthetic one~\citep{brakel2017learning} with 6 sources among which 2 are noise, and a real audio dataset~\citep{kabal2002tsp} with 3 sources among which one is noise, also used in~\cite{brakel2017learning}. The evaluation metric is the maximum correlation between true and reconstructed sources. This metric is not available without ground truth hence we select the model with lowest dHSIC for $\mX$, which can be exactly evaluated since the number of sources is small.
\begin{figure*}
\centering
\includegraphics[width=.7\textwidth, trim={0 3cm 7cm 12.3cm},clip]{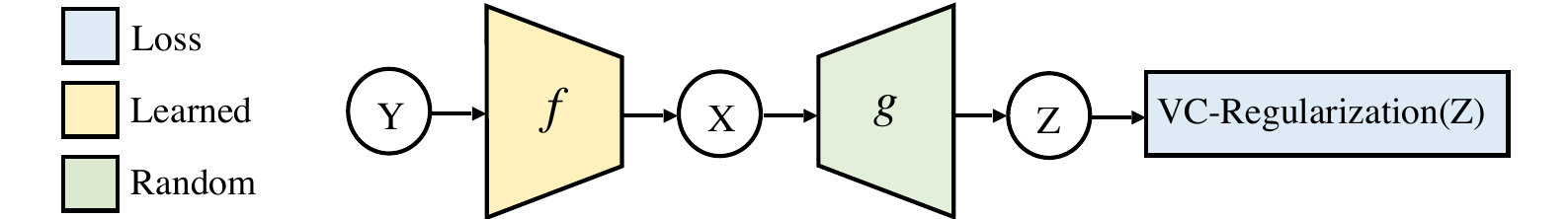}\\[-0.5em]
\caption{\small Linear ICA model expressed from VCReg of the projector's output.}
\label{fig:archi_linearica}
\end{figure*}

\begin{table*}
    \centering
    \small
    \caption{\small Maximum correlation between true and reconstructed sources (the higher the better $\uparrow$). The projector performs competitive reconstruction in the linear setting (left), where pairwise independence is sufficient, but not in PNL (right) where it is not.}
    \begin{tabular}{l|c|c|c|c|c} \toprule
    \multicolumn{3}{c|}{Linear mixture} & \multicolumn{3}{c}{Post Non Linear mixture} \\
    \midrule
    {Method} & {Synthetic data} & {Audio data} & {Method} & {Synthetic data} & {Audio data}\\
    \midrule
    {Whitening} & 0.8074 & 0.9876 & {Whitening} & 0.7981 & 0.9046 \\ 
    {FastICA}  & 0.9998 &  1.0 & {FastICA} &  0.8311 & 0.8989 \\
    {Anica}  & 0.9987$\pm$6.5e-4 & 0.9996$\pm$4.9e-4 & {Anica} & 0.9794$\pm$53e-4 & 0.9929$\pm$18e-4 \\
    {VCReg}  & 0.9986$\pm$8.2e-4 & 0.9936$\pm$64e-4 & {VCReg}  & 0.8465$\pm$142e-4 & 0.8706$\pm$376e-4 \\
    \bottomrule
    \end{tabular}
    \label{tab:max_corr}
\end{table*}

Our model recovers the sources, and is even competitive with methods specifically designed for linear ICA such as Fast ICA~\citep{hyvarinen2000independent} or Anica~\citep{brakel2017learning}, see Table~\ref{tab:max_corr}. Both increasing the width of the projector and resampling it at each step (especially for smaller projectors) improve the reconstruction as can be seen in Figure~\ref{fig:ica}. This is in line with our findings in Section~\ref{sec:result} and Experiment~\ref{exp:val_theory}.
ICA experiments done with a learned projector fail to recover independent sources: this further demonstrates that learning the projector to optimize VCReg is counter-productive, and that in the context of SSL, learning the projector is rather useful for satisfying the invariance criterion, which is absent in the ICA experiments. VCReg is rather optimized by the encoder $f$.

\begin{figure}
\begin{center}
\centerline{\includegraphics[width=.85\columnwidth]{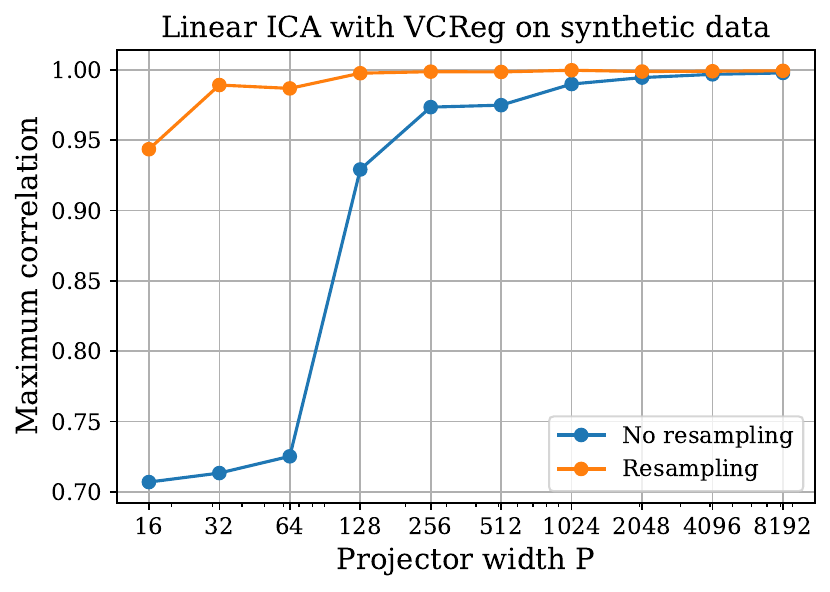}}
\caption{Resampling and increased width of the projector $g$ both improve the quality of the source reconstruction measured by the maximum correlation between true and reconstructed sources ($\uparrow$).}
\label{fig:ica}
\end{center}
\vspace{-.2cm}
\end{figure}
\vspace{-.2cm}
\paragraph{VC regularized projectors do not solve nonlinear ICA.}
\cref{fig:indep_training} suggests that mutual independence is optimized during training although not properly: one could ask whether VCReg also enforces it enough to solve nonlinear ICA. To test this hypothesis, we apply VCReg to a particular case of nonlinear ICA which allows identifiability but does not have equivalence between pairwise and mutual independence: the post-nonlinear mixture (PNL)~\citep{taleb1999source}. In PNL, the sources are linearly mixed before being fed to elementwise nonlinear functions. For these experiments, and following~\cite{brakel2017learning}, our encoder is a MLP. During our first experiments, we observed an informational collapse of the encoder, which produces seemingly mutually independent variables with very poor reconstruction of the sources. To alleviate this issue, we add a decoder taking $\mX$ as input and reconstructing $\mY$. Figure~\ref{fig:archi_pnlica} in Appendix~\ref{app:exp_details} shows our modified setup. We compare VCReg to FastICA and Anica. Although FastICA is not meant to solve the nonlinear case, it remains an interesting baseline. Table~\ref{tab:max_corr} shows that our model fails to recover the sources, as it does only slightly better than FastICA in the synthetic case. %
Hence, although mutual independence increases during training, VCReg does not optimize it enough to solve nonlinear ICA. We propose a simple explanation for this limitation. Indeed, each feature in $\mZ$ is a nonlinear function of all features in $\mX$. Hence, it is not possible to completely decorrelate two components of $\mZ$ as they both contain the same set of features. It is still possible to improve mutual independence to some extent since, in practice, only parts of the inputs are considered at once by nonlinear mappings such as neural networks~\citep{erhan2009visualizing,adebayo2018sanity}.

\vspace{-.3cm}
\section{Conclusion}
\vspace{-.1cm}
This work provides the first theoretical explanation of SSL projectors by demonstrating they enforce pairwise independence in the learned representation. This quantity is correlated to in-domain and out-of-domain downstream accuracy and could be further exploited for model selection without labels. The classical evaluation of SSL being linear probing, further work could also explore whether pairwise independence is connected to linear separability of the features. Finally, this work focuses on random weights with uniform distribution yet opens the way to weight distributions that are closer to training parameters in practice which, for example, have been characterized for over-parameterized networks~\citep{jacot2018neural}.

\newpage
\newpage

\bibliography{icml2023}
\bibliographystyle{icml2023}

\newpage
\appendix
\onecolumn

\section{Further Background}
\label{app:background}

\subsection{VCReg methods}

\paragraph{Learning visual representations} has a long history in machine learning. Shortly after the advent of convolutional neural networks (CNN), it was common to train a model on a supervised task such as ImageNet before removing the classifier and use the remaining model to produce features for downstream tasks (e.g., classification, segmentation), a technique coined transfer learning. Then, strategies to learn representations without labelled dataset by enforcing invariant embeddings for different versions of the same sample emerged~\citep{misra2020self, chen2020simple}, hence the term ``joint embedding''. Joint embedding methods can be divided in two categories. Contrastive learning (e.g., SimCLR~\citep{chen2020simple}) pulls together representations of two augmented versions of the same image while pushing away the representation of this image from the representation of different images. Non-contrastive learning (e.g., DINO~\citep{caron2021emerging}) pulls together representations of two augmented versions of the same image while avoiding collapse using different techniques. The frameworks considered in this work belong to the latter category. Non-contrastive self-supervised representation learning of images has a few peculiarities:
\begin{itemize}
    \item Data augmentations are central, at least when it comes to learn from ImageNet, and must be carefully chosen.
    \item Most methods are typically used either with Vision Transformers~\citep{dosovitskiy2020image}, or with CNNs such as ResNets~\citep{he2016deep}. In this work, we focus on ResNets.
    \item Model selection is usually performed via the Top1 accuracy on the validation set of ImageNet using a linear classifier on top of the learned representation. Outside of joint embedding, methods such as Masked-Auto-Encoder~\citep{he2021masked} perform poorly in linear evaluation while delivering excellent performance when fine-tuned on the downstream task.
\end{itemize}

\paragraph{The crucial role of the projector.} It was found by~\cite{chen2020simple} that adding a MLP on top of the encoder (removed after the training) significantly improved the quality of the learned representation. For example, in 100 epochs, VICReg without projector would achieve $48\%$ validation accuracy on ImageNet instead of $68$ and SimCLR $50\%$ instead of $68\%$. Since then, a few theoretical work attempted to explain SSL with joint embeddings. \cite{jing2021understanding} study the role of linear projectors with restricted augmentations, while other works such as~\citep{huang2021towards,wang2020understanding,haochen2021provable,tian2020makes, wang2021understanding} only consider an encoder without projector in their theoretical analysis. Finally, \cite{cosentino2022toward} consider a MLP projector in a very restricted context where data augmentations are Lie group transformations.

\paragraph{BarlowTwins.}~\citet{zbontar2021barlow} propose a slightly different approach based on regularizing $\mC$, the $P \times P$ cross-correlation matrix between $\mZ_{\rm left}$ and $\mZ_{\rm right}$ by optimizing
\begin{align}
    \mathcal{L}_{\rm BT}=\sum_{k=1}^{K}((\mC)_{k,k}-1)^2+\alpha \sum_{k' \not = k}(\mC)_{k,k'}^2\;.\label{eq:BT}
\end{align}
Here, the leftmost term corresponds to regularizing the cross-correlation of the same feature in the two views to be unit, while the rightmost term regularizes the cross-correlation between pairs of different features in the two views. Importantly, $(\mC)_{i,j}$ falls back to measuring the cosine similarity between the $i^{\rm th}$ column of $\mZ_{\rm left}$ and the $j^{\rm th}$ column of $\mZ_{\rm right}$ i.e. $(\mC)_{i,j}=\frac{\langle (\mZ_{\rm left})_{.,i},(\mZ_{\rm right})_{.,j}\rangle}{\|(\mZ_{\rm left})_{.,i}\|_2\|(\mZ_{\rm right})_{.,j}\|_2}$. %
Hence, the leftmost term is also an invariance term: all features must be similar for both views. 

\paragraph{W-MSE.} \citet{ermolov2021whitening} use the following loss
\begin{align*}
    \min \;& 2 - 2 \frac{\langle \mZ_{\rm left} , \mZ_{\rm right} \rangle}{\| \mZ_{\rm left} \|_2 \dot \| \mZ_{\rm right} \|_2}\; \\
    &\text{s.t. } \Cov(\mZ_{\rm left}) = Id, \Cov(\mZ_{\rm right}) = Id.
    \label{eq:WMSE}
\end{align*}
where the cosine similarity can be replaced by the Euclidean distance.

\subsection{Independent Component Analysis}

The goal of ICA~\citep{comon1994independent} is to find a transformation of a random vector $\mY$\footnote{For simplicity, $\mY$ directly denotes $N$ empirical $D$-dimensional observations.} which minimizes the statistical dependence of its components. Its simplest instance is linear ICA, and is motivated by problems such as the cocktail-party, where $\mY \in \mathbb{R}^{N \times D}$ typically results from a linear transformation of independent sources $\mS$ (\textit{e.g} overlapping voices) one wants to recover. Formally, $\mY = \mS\mA$, with $\mA \in \mathbb{R}^{D \times D}$ an unknown mixing matrix that can therefore not be inverted. Instead, the linear ICA approach searches for $\mM \in \mathbb{R}^{D \times D}$ such that $\mY\mM=\mS$ by maximizing the statistical \textit{mutual} independence between the components of $\mY\mM$. Being able to recover $\mS$ is a property known as identifiability.

\section{Proofs}

\subsection{Proof of \cref{prop:linar_proj}: Random Linear Projectors maximize pairwise independence}

\label{app:proofs}
\begin{proof}

Let's consider the linear regime for which $g((\mX)_{:,i})=(\mX)_{:,i}\vw^T_{i}$ with orthogonal weights i.e. $\langle \vw_i,\vw_j\rangle=1_{\{i=j\}}$. This assumption is realistic since we assume that $\vw \in \mathbb{R}^{K}$ with $K$ very large, and that those are randomly sampled. From that, we then have
\begin{align*}
    \sum_{i\not = j}\HSIC((\mX)_{:,i}, (\mX)_{:,j})
    =&\frac{1}{(N-1)^2}\sum_{i\not = j}\Tr(g((\mX)_{:,i})g((\mX)_{:,i})^T\mH g((\mX)_{:,j})g((\mX)_{:,j})^T\mH)\\
    =&\frac{1}{(N-1)^2} \sum_{i\not = j}\|g((\mX)_{:,i})^T\mH g((\mX)_{:,j}) \|_F^2,
\end{align*}
we will now push the sum inside the norm by considering the following equality:
\begin{align*}
    \| \sum_{i\not = j}f(i,j)\|_F^2=&\sum_{i\not = j}\| f(i,j)\|_F^2 + \sum_{i\not = j}\sum_{k \not = \ell ,(i,j)\not = (k,\ell)} \Tr\left(f(i,j)^Tf(k,\ell)\right)=\sum_{i\not = j}\| f(i,j)\|_F^2,
\end{align*}
since in our case
\begin{align*}
  f(i,j)&=g((\mX)_{:,i})^T\mH g((\mX)_{:,j})=\vw_{i}(\mX)_{:,i}^T\mH(\mX)_{:,j}\vw_{j}^T,\\
  f(i,j)^Tf(k,\ell)&=
  \vw_{j}(\mX)_{:,j}^T\mH(\mX)_{:,i}\vw_{i}^T
  \vw_{k}(\mX)_{:,k}^T\mH(\mX)_{:,\ell}\vw_{\ell}^T
  ,\\
  \Tr(f(i,j)^Tf(k,\ell))&=1_{\{i = k \wedge j = \ell\}}(\mX)_{:,j}^T\mH(\mX)_{:,i}(\mX)_{:,k}^T\mH(\mX)_{:,\ell},
\end{align*}
leading to
\begin{align*}
    \sum_{i\not = j}\HSIC((\mX)_{:,i}, (\mX)_{:,j})
    =&\frac{1}{(N-1)^2}\|\sum_{i\not = j} g((\mX)_{:,i})^T\mH g((\mX)_{:,j}) \|_F^2\\
    &\hspace{-2cm}=\frac{1}{(N-1)^2}\|(\sum_{i} g((\mX)_{:,i}))^T\mH (\sum_{j}g((\mX)_{:,j}))-\sum_{i}g((\mX)_{:,i})^T\mH g((\mX)_{:,i}) \|_F^2\\
    = &\|\frac{1}{N-1}(\sum_{i} g((\mX)_{:,i}))^T\mH (\sum_{j}g((\mX)_{:,j}))-\diag(\Cov(\mX\mW)) \|_F^2\\
    = &\|\Cov(\mX\mW)-\diag(\Cov(\mX\mW)) \|_F^2\\
    =&\sum_{i \not = j}\Cov(\mX\mW)_{i,j}^2,
\end{align*}
since we assume that all columns of $\mX$ have same variance and that $\mW$ is orthogonal we have for the pre-last equality
\begin{align*}
    \frac{1}{N-1}\sum_{i}g((\mX)_{:,i})^T\mH g((\mX)_{:,i})=\Var((\mX)_{:,1})\sum_{i}\vw_i\vw_i^T
    =\Var((\mX)_{:,1}),
\end{align*}
and for the last equality we use the fact that since $g$ is linear, $\sum_{i} g((\mX)_{:,i})=\sum_{i}(\mX)_{:,i}\vw_i^T=\mX\mW$ with $\mW=[\vw_1,\dots,\vw_K]^T$.

\end{proof}

\subsection{On the equivalence between VICReg and BarlowTwins objectives}
\label{proof:BT}

We can express Barlow Twins objective as:
\begin{align*}
    \min &\sum_{k=1}^{K}(\Cov(\mZ_{\rm left}, \mZ_{\rm right})_{k,k}-1)^2+\alpha \sum_{k' \not = k}\Cov(\mZ_{\rm left}, \mZ_{\rm right})_{k,k'}^2\; \\
    &\text{s.t. } \Cov(\mZ_{\rm left}) = \mI, \Cov(\mZ_{\rm right}) = \mI.
\end{align*}
Assuming $\mZ_{\rm left}^T\mZ_{\rm left}=\mI,\mZ_{\rm right}^T\mZ_{\rm right}=\mI$ i.e. perfect minimization of the variance and covariance terms, we have
\begin{align*}
    C_{i,j}=&\frac{\langle (\mZ_{\rm left})_{:,i},(\mZ_{\rm left})_{:,j}\rangle}{\|(\mZ_{\rm left})_{:,i} \|_2 \|(\mZ_{\rm left})_{:,j} \|_2}=\frac{1}{2}\langle (\mZ_{\rm left})_{:,i},(\mZ_{\rm right})_{:,j}\rangle=-\| (\mZ_{\rm left})_{:,i} - (\mZ_{\rm right})_{:,j}\|_2^2-1,
\end{align*}
and thus
\begin{align*}
    \sum_{i}(C_{i,i}-1)^2=\sum_i\| (\mZ_{\rm left})_{:,i} - (\mZ_{\rm right})_{:,i}\|_2^4=\|\mZ_{\rm left} - \mZ_{\rm right}\|_F^4\propto I(\mZ_{\rm left},\mZ_{\rm right}),
\end{align*}
so we recover the invariance loss exactly with the diagonal terms of BarlowTwins. We now show this is actually enough to minimize this quantity to also minimize the off-diagonal terms: %
\begin{align*}
    \sum_{i\not = j}C_{i,j}^2=&\sum_{i\not = j}(1-\frac{1}{2}\| (\mZ_{\rm left})_{:,i}-(\mZ_{\rm right})_{:,j}\|_2^2)^2\\
    =&\sum_{i\not = j}(1-\frac{1}{2}\| (\mZ_{\rm left})_{:,i}-(\mZ_{\rm left})_{:,j}+(\mZ_{\rm left})_{:,j}-(\mZ_{\rm right})_{:,j}\|_2^2)^2\\
    =&\sum_{i\not = j}\Big(1-\frac{1}{2}\| (\mZ_{\rm left})_{:,i}-(\mZ_{\rm left})_{:,j}\|_2^2-\frac{1}{2}\| (\mZ_{\rm left})_{:,j}-(\mZ_{\rm right})_{:,j}\|_2^2\\
    &-\langle (\mZ_{\rm left})_{:,i}-(\mZ_{\rm left})_{:,j},(\mZ_{\rm left})_{:,j}-(\mZ_{\rm right})_{:,j}\rangle \Big)^2\\
    =&\sum_{i\not = j}\Big(-\frac{1}{2}\| (\mZ_{\rm left})_{:,j}-(\mZ_{\rm right})_{:,j}\|_2^2-\langle (\mZ_{\rm left})_{:,i}-(\mZ_{\rm left})_{:,j},(\mZ_{\rm left})_{:,j}-(\mZ_{\rm right})_{:,j}\rangle \Big)^2\\
    =&\sum_{i\not = j}\Big(\frac{1}{2}\| (\mZ_{\rm left})_{:,j}-(\mZ_{\rm right})_{:,j}\|_2^2+\langle (\mZ_{\rm left})_{:,i}-(\mZ_{\rm left})_{:,j},(\mZ_{\rm left})_{:,j}-(\mZ_{\rm right})_{:,j}\rangle \Big)^2\\
    \leq&\sum_{i\not = j}\Big(\frac{1}{2}\| (\mZ_{\rm left})_{:,j}-(\mZ_{\rm right})_{:,j}\|_2^2+\| (\mZ_{\rm left})_{:,i}-(\mZ_{\rm left})_{:,j}\|_2^2 \|(\mZ_{\rm left})_{:,j}-(\mZ_{\rm right})_{:,j}\|_2^2 \Big)^2\\
    = &\frac{25(K-1)}{4}\| \mZ_{\rm left}-\mZ_{\rm right}\|_2^4,
\end{align*}
hence minimizing the BarlowTwins loss with the explicit whitening constraint is equivalent to minimizing VICReg with explicit whitening constraint.

\clearpage
\section{Additional experimental results}
\label{app:add_exps}

\subsection{Illustrating Theorem~\ref{cor:composition}}

\begin{figure}[h!]
    \centering
    \includegraphics[scale=.45]{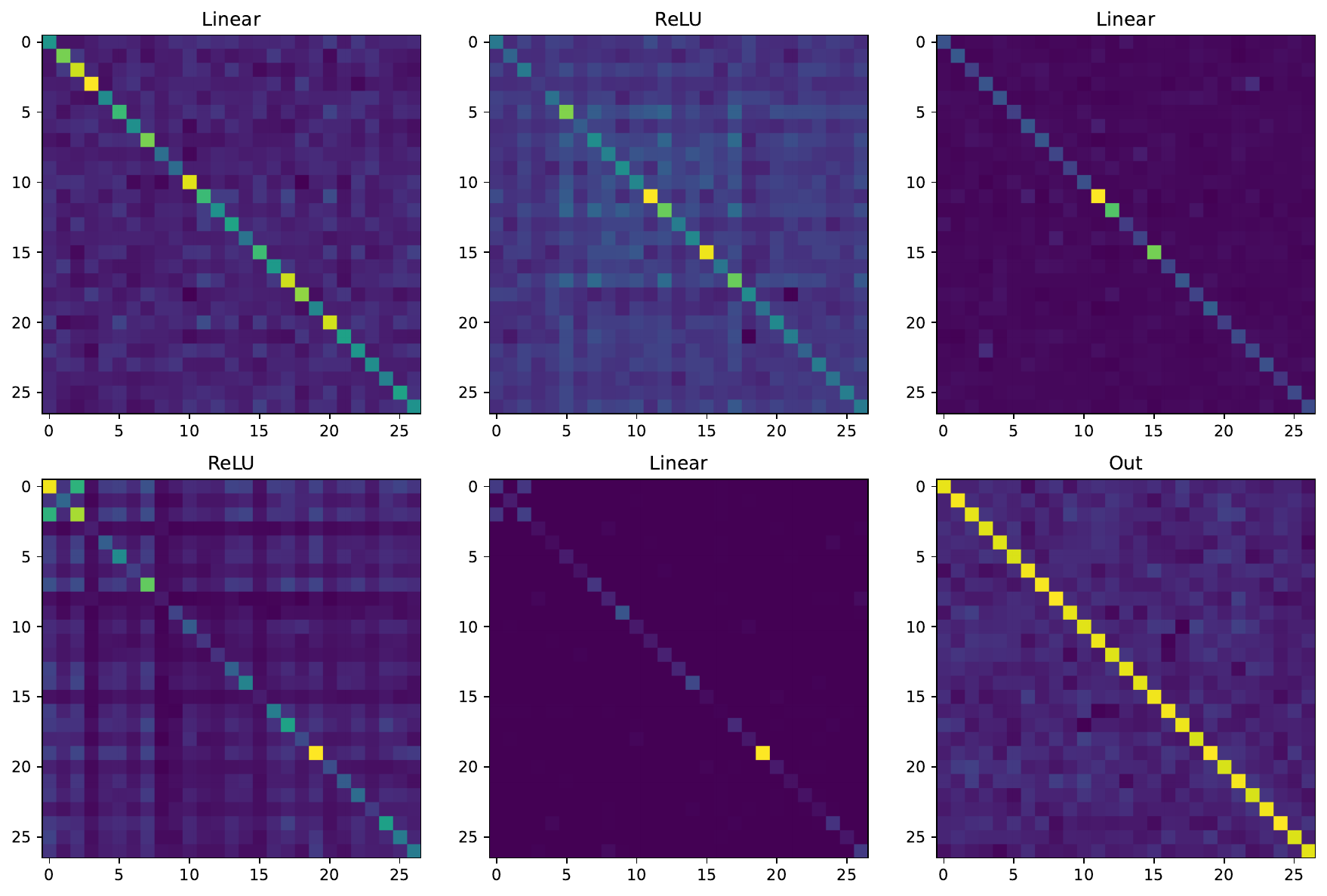}
    \caption{The covariance matrices for some features of each activation in a learned VICReg projector with width $8192$ before corresponding hidden layer are close to diagonal, suggesting that each hidden layer output is implicitly VC-regularized.
    }
    \label{fig:implicit_reg}
\end{figure}
\begin{figure}[h!]
    \centering
    \includegraphics[scale=.45]{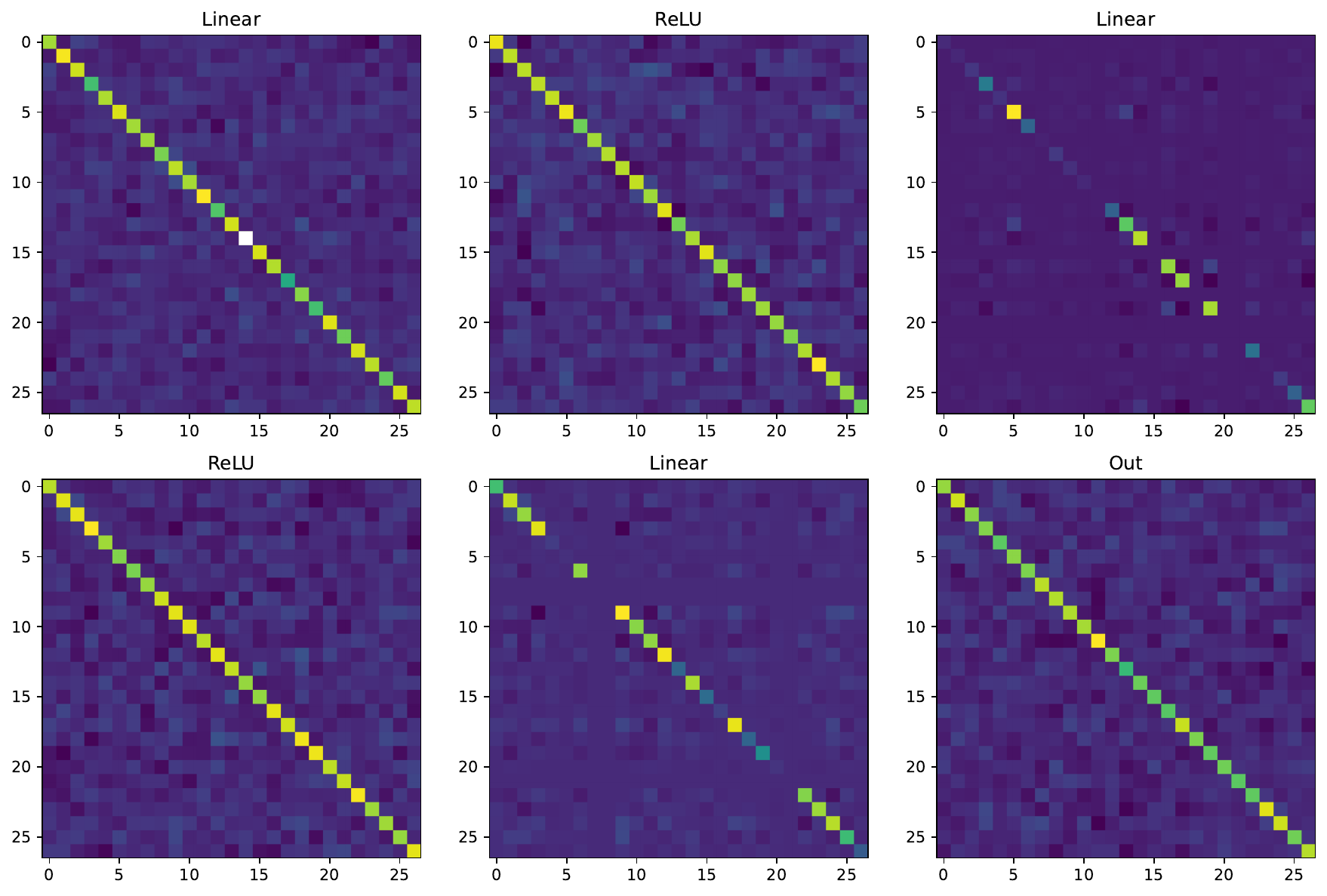}
    \caption{The covariance matrices for some features of each activation in a random VICReg projector (with width $8192$) before corresponding hidden layer are close to diagonal, suggesting that each hidden layer output is implicitly VC-regularized.
    }
    \label{fig:implicit_reg_random}
\end{figure}

\newpage

\subsection{Other factors of variation for HSIC}
\label{app:more_hsic}

In this subsection, we illustrate the variation of HSIC induced by two more factors: varying the depth of the projector (~\cref{fig:hsic_vs_depth}), and the covariance coefficient in VCReg (~\cref{fig:cov_coeff_hsic}).

\begin{figure}[h]
    \centering
    \includegraphics[width=.65\columnwidth]{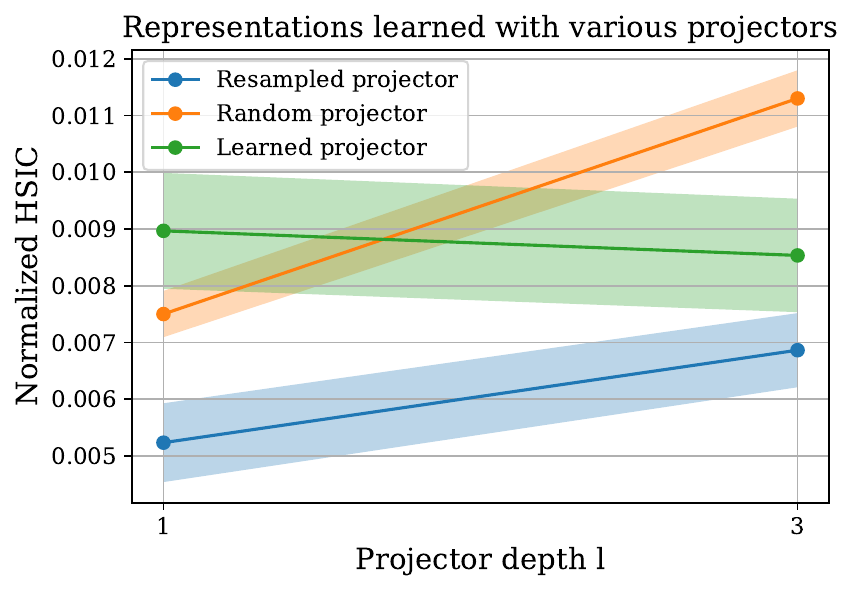}
    \caption{Increasing depth hurts HSIC for random projectors. Resampled projectors yield the representations with the lowest dependence amount.}
    \label{fig:hsic_vs_depth}
\end{figure}

\begin{table}[h]
    \small
    \centering
    \caption{VICReg with a 512-512-512 projector and varying covariance coefficient. Increasing this coefficient decreases HSIC. Best accuracy does not always correspond to lowest Normalized HSIC: optimizing too much for Normalized HSIC is detrimental to the information content of the representation (it is trivial to train a representation too look pairwise independent. Doing so while extracting information from the input is harder).}
    \begin{tabular}{lccccc} \toprule
    {Covariance coefficient in VCReg} & 1 & 2 & 4 & 8 & 16 \\
    \midrule
    {Normalized HSIC ($\downarrow$)} & 0.017 & 0.0119 & 0.0090 & 0.0084 & 0.0084 \\
    \midrule
    {Top1 ImageNet ($\uparrow$)} & 60.8 & 62.5 & 63.5 & 63.1 & 62.4\\
    \bottomrule
    \end{tabular}
\label{fig:cov_coeff_hsic}
\end{table}

\clearpage
\section{Implementations}
\label{app:implem}

\subsection{HSIC~\cref{eq:HSIC}}
\begin{lstlisting}[language=Python,escapechar=\%]
def GaussianKernelMatrix(X, sigma):
    pairwise_distances = torch.cdist(X, X)
    return torch.exp( -pairwise_distances / (2 * sigma**2))

def HSIC(X_1, X_2, sigma_1, sigma_2):
    N = x.size(0) # batch size, should be the same for X_1 and X_2
    K_1 = GaussianKernelMatrix(X_1, sigma_1) # Gaussian kernel matrix of X_1 with bandwidth sigma_1
    K_2 = GaussianKernelMatrix(X_2, sigma_2) # Gaussian kernel matrix of X_2 with bandwidth sigma_2
    H = torch.eye(N) - 1.0 / N  # centering matrix
    HSIC = torch.trace(K_1 @ H @ K_2 @ H)/((N-1)**2)
    return HSIC

\end{lstlisting}

In our study, the bandwidth $\sigma$ of the Gaussian kernel is determined by the median of the distribution of pairwise euclidean distances between samples.

\subsection{dHSIC~\cref{eq:dHSIC_kernel}}
dHSIC can be estimated given empirical samples $\mX_1,\dots, \mX_d$ with respective kernel matrices $\mK_1,\dots, \mK_d$  as:
\begin{align*}
    \dHSIC(\mX_1,\dots, \mX_d) =& \frac{1}{N^2} \sum_{i, j} (\bigodot_{k=1}^d \mK_k)_{i,j} + \frac{1}{N^{2d}} \prod_{k=1}^d \sum_{i,j} (\mK_k)_{i,j} - \frac{2}{N^{d + 1}} \sum_{i} \bigodot_{k=1}^d \sum_j (\mK_k)_{i,j}.
\end{align*}
\label{eq:dHSIC}
When $d=2$, the first term corresponds to a biased $\HSIC$ estimator. The implementation of dHSIC is given by:
\vspace{.5cm}
\begin{lstlisting}[language=Python,escapechar=\%]
import GaussianKernelMatrix # defined above

def dHSIC(X, sigma):
    length = X.shape(0)
    term_1 = 1.0
    term_2 = 1.0
    term_3 = 2.0 / length
    for j in range(D):
        K_j = GaussianKernelMatrix(X[j])
        term_1 = torch.mul(term_1, K_j)
        term_2 = 1.0 / length / length / term_2 * torch.sum(K_j)
        term_3 = 1.0 / length * term_3 * K_j.sum(axis=0)
    
    term_1 = (1.0 / length) ** 2 * torch.sum(term_1)
    term_3 = torch.sum(term_3)

    return term_1 + term_2 - term_3 # the three terms of the estimator
    
\end{lstlisting}

As these evaluations already take a few minutes, scaling it to significantly larger portions of components would be impractical.

\subsection{Linear ICA model~\ref{fig:archi_linearica}}
\begin{lstlisting}[language=Python,escapechar=\%]
for p in projector.parameters(): # freeze the projector
        p.requires_grad = False
        
for y in loader:
    x = encoder(y) # matrix multiplication with M
    
    z = projector(x) # compute embedding
    
    C = torch.cov(z) # covariance matrix
    loss = torch.MSE(C, torch.identity()) # VCReg loss
    loss.backward()
    optimizer.step()
    
    projector.__init__()  # to resample the projector
    for p in projector.parameters(): # freeze the projector
        p.requires_grad = False
\end{lstlisting}

\subsection{Nonlinear ICA model~\ref{fig:archi_pnlica}}
The implementation of our nonlinear ICA model~\ref{fig:archi_pnlica} differs from~\ref{fig:archi_linearica} by the addition of a reconstruction constraint:
\vspace{.5cm}
\begin{lstlisting}[language=Python,escapechar=\%]
for p in projector.parameters(): # freeze the projector
        p.requires_grad = False
        
for y in loader:
    x = encoder(y) # MLP encoder
    
    z = projector(x) # compute embedding
    
    y_rec = reconstructor(x) # MLP reconstructor
    
    C = torch.cov(z) # covariance matrix
    loss = torch.MSE(C, torch.identity()) + lmda * torch.MSE(y, y_rec) # VCReg loss with reconstruction
    loss.backward()
    optimizer.step()
    
    projector.__init__()  # to resample the projector
    for p in projector.parameters(): # freeze the projector
        p.requires_grad = False
\end{lstlisting}

\clearpage
\section{Experimental details}
\label{app:exp_details}

\subsection{Details on ImageNet experiments}

\paragraph{Architectures and hyper-parameters.} We use a Resnet50 bottleneck, optimized during $100$ epochs with LARS and an initial learning rate of $0.3$ for all methods. The batch size is $1024$ for all methods except SimCLR. All methods have projector $8192-8192-8192$ except DINO and Supervised. Throughout the training, we evaluate the learned representation using an online linear classifier. Details specific to each method:
\begin{itemize}
    \item Barlow Twins: for Experiment~\ref{exp:analysis}, the off-diagonal coefficient is $0.0051$.
    \item VICReg: for Experiment~\ref{exp:analysis}, the projector size is $8192-8192-8192$, Invariance coeff is $25$, Variance coeff is $25$ and is Covariance coeff $1$. For the rest of the experiments, and following~\cite{bardes2021vicreg}, we only move the covariance coefficient when changing the size of the projector. We scale it in the square root of the output size of the projector in order to have similar magnitude for the covariance term for all projector sizes.
    \item SimCLR: for Experiment~\ref{exp:analysis}, we choose a higher batch size of $2048$ as SimCLR is sensitive to this parameter~\citep{chen2020simple}. The temperature is $0.15$.
    \item DINO: we use $8$ crops. The head has $4$ layers of widths $2048-2048-2048-256-65536$.
\end{itemize}

\paragraph{Augmentations.} At train time, the resolution is $160$ and we use
\begin{itemize}
    \item $\texttt{RandomHorizontalFlip()}$.
    \item $\texttt{ColorJitter(0.8, 0.4, 0.4, 0.2, 0.1)}$.
    \item $\texttt{Greyscale(0.2)}$.
    \item $\texttt{NormalizeImage}$ with ImageNet mean and standard deviation.
    \item $\texttt{GaussianBlur()}$ with kernel size $(5, 9)$ and sigma $(0.1, 2)$.
\end{itemize}
At validation time, the resolution is $224$ and the images are cropped and normalized with ImageNet mean and standard deviation.
        
\subsection{ICA setup}

\paragraph{Detailed setup.} In these experiments, we keep the same optimizer, train on $100$ epochs and choose a batch size of $64$ for both datasets. We tune the learning rate according to a logarithmic grid $[1.0, 10.0, 100.0]$ (then rescaled according to $\frac{lr \times batch size}{256}$). The MLP can either be a SSL-like projector or simply a fully-connected layer followed by a ReLU. The wider the MLP, the better the result hence it does not require selection. For the synthetic dataset, the MLP has width $1024$, and $8192$ for the audio one.
\begin{itemize}
    \item Linear ICA: We tune the standard and covariance coefficients according to a logarithmic grid $[1, 10, 100]$.
    \item PNL ICA: Our architecture for the nonlinear ICA experiments is presented in Figure~\ref{fig:archi_pnlica}. The encoder $f$ is a MLP with $3$ layers and width $128$. The decoder $h$ is a learnable nonlinearity (a MLP with $3$ hidden layers and depth $16$) followed by a fully-connected layer. We tune the standard, covariance and reconstruction coefficients according to a larger logarithmic grid $[1, 3, 10, 30, 100]$.
\end{itemize}

\begin{figure}
    \centering
    \includegraphics[scale=.4, trim={16cm .9cm 0 1cm},clip]{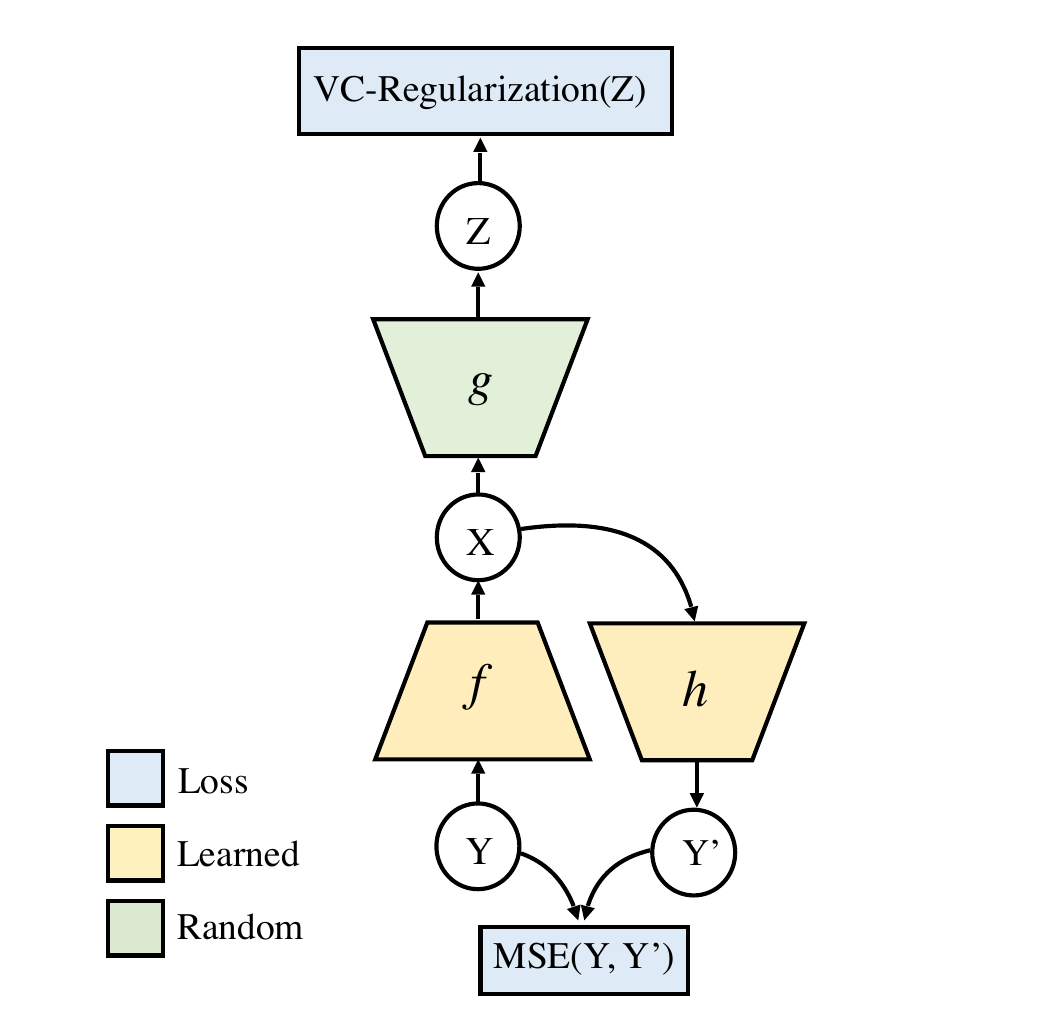}
    \caption{Nonlinear ICA model.}
    \label{fig:archi_pnlica}
\end{figure}

\paragraph{Reconstructed sources for ICA.} In this paragraph, we provide ground truth sources for the synthetic data along with examples of reconstructed sources. 

\begin{figure}
    \centering
    \includegraphics[scale=.7]{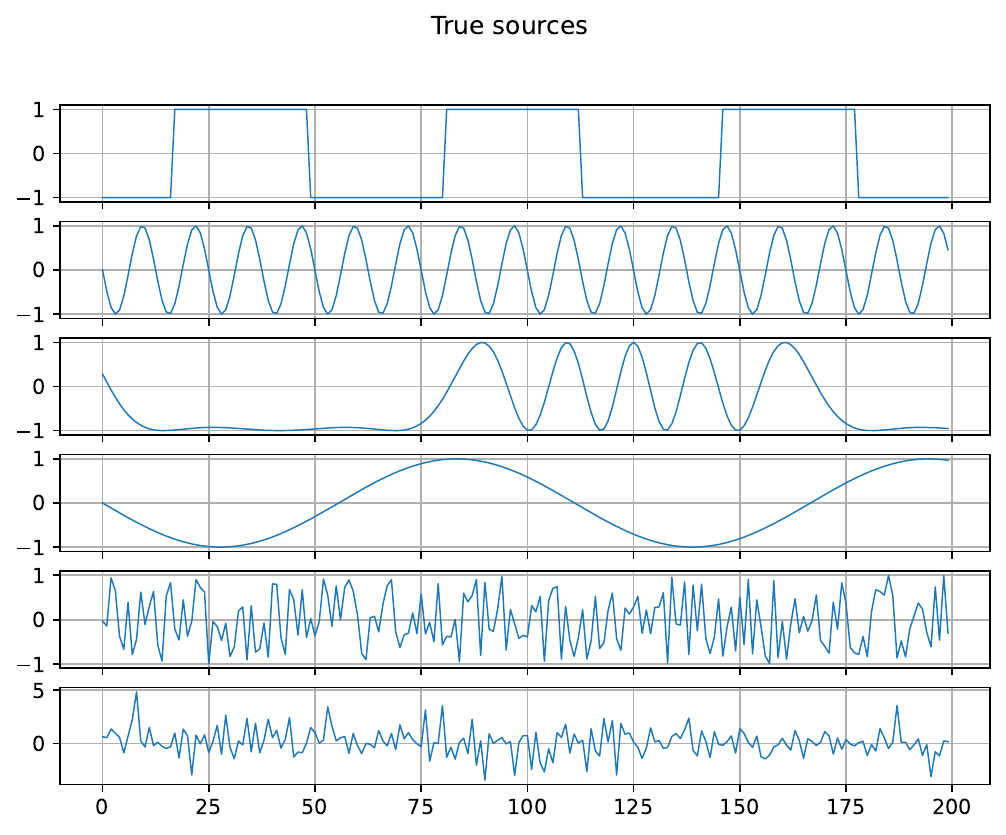}
    \label{fig:true_sources}
    \caption{Data before mixing.}
\end{figure}
\begin{figure}
    \centering
    \includegraphics[scale=.7]{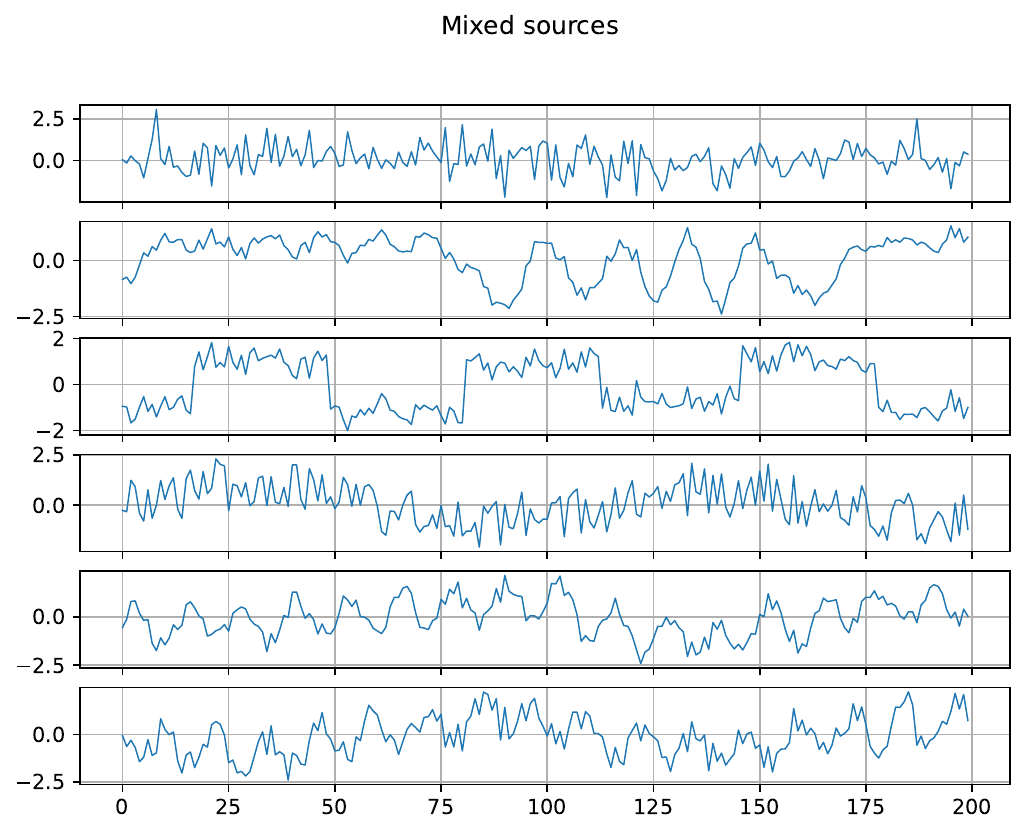}
    \label{fig:mixed_sources}
    \caption{Data after mixing, before feeding to the various ICA models.}
\end{figure}
\begin{figure}
    \centering
    \includegraphics[scale=.7]{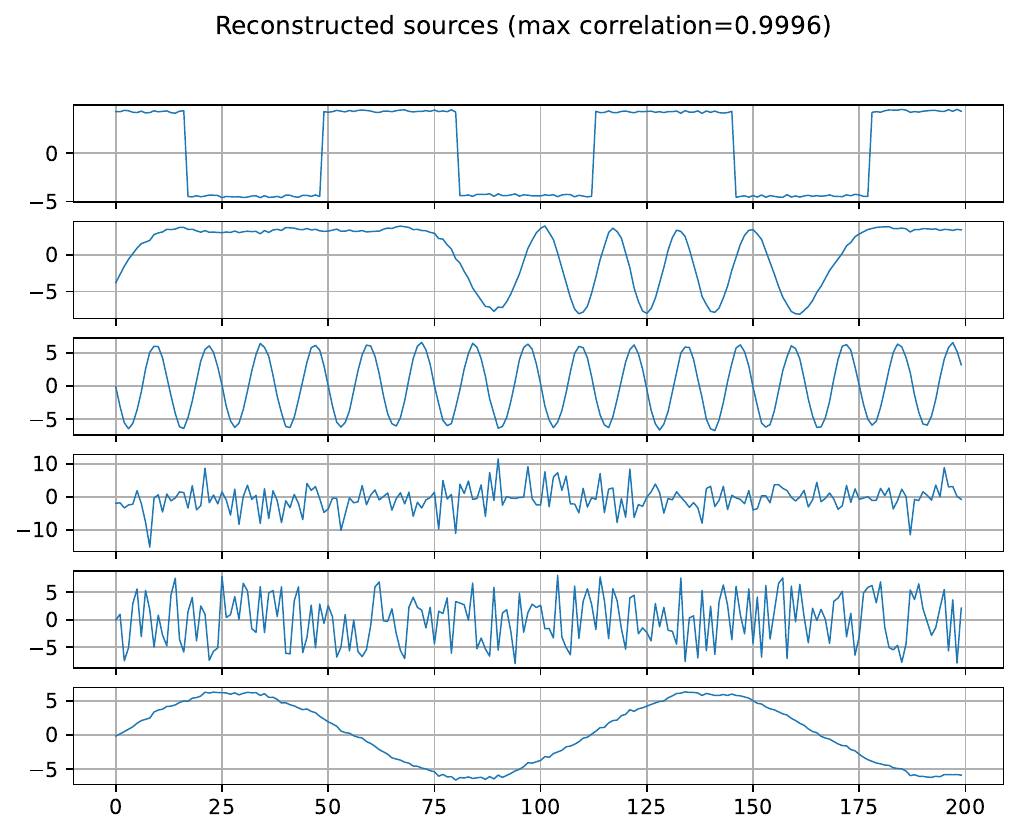}
    \caption{This level of max correlation is typically achieved by FastICA, VCReg or Anica.}
    \label{fig:rec_sources_ours}
\end{figure}
\begin{figure}
    \centering
    \includegraphics[scale=.7]{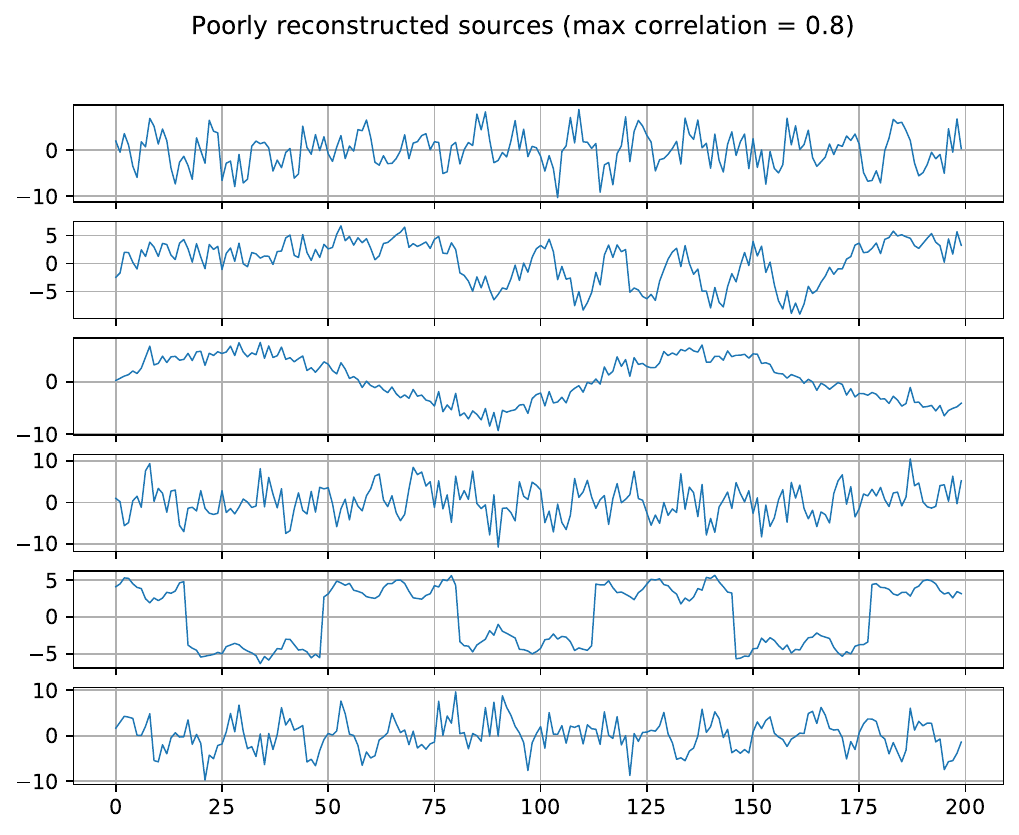}
    \caption{This level of max correlation is typically achieved by Whitening.}
    \label{fig:rec_sources_whitening}
\end{figure}

\end{document}